\documentclass{article}

\usepackage{microtype}
\usepackage{graphicx}
\usepackage{subfigure}
\usepackage{amsmath}
\usepackage{booktabs} 

\usepackage{balance} 
\usepackage{multicol}
\usepackage{subfigure}
\usepackage{amssymb}
\usepackage{amsthm}
\usepackage{multirow}
\usepackage{paralist}

\usepackage{bbm}
\usepackage{xspace}

\def\R{\mathbb{R}}

\def\R{\mathbb{R}}

\newtheorem{theorem}{Theorem}
\newtheorem{lemma}{Lemma}
\newtheorem{corollary}{Corollary}

\usepackage{hyperref}

\usepackage[accepted]{icml2018}

\begin{document}

\twocolumn[
\icmltitle{SQL-Rank: A Listwise Approach to Collaborative Ranking}



\icmlsetsymbol{equal}{*}


\begin{icmlauthorlist}
\icmlauthor{Liwei Wu}{stat,cs}
\icmlauthor{Cho-Jui Hsieh}{stat,cs}
\icmlauthor{James Sharpnack}{stat}
\end{icmlauthorlist}

\icmlaffiliation{cs}{Department of Computer Science, University of California, Davis, CA, USA}
\icmlaffiliation{stat}{Department of Statistics, University of California, Davis, CA, USA}

\icmlcorrespondingauthor{Liwei Wu}{liwu@ucdavis.com}

\icmlkeywords{Collaborative ranking, recommendation systems, listwise ranking}

\vskip 0.3in
]



\printAffiliationsAndNotice{}  

\begin{abstract}
In this paper, we propose a listwise approach for constructing user-specific rankings in recommendation systems in a collaborative fashion.
We contrast the listwise approach to previous pointwise and pairwise approaches, which are based on treating either each rating or each pairwise comparison as an independent instance respectively.
By extending the work of \cite{cao2007learning}, we cast listwise collaborative ranking as maximum likelihood under a permutation model which applies probability mass to permutations based on a low rank latent score matrix.
We present a novel algorithm called SQL-Rank, which can accommodate ties and missing data and can run in linear time. 
We develop a theoretical framework for analyzing listwise ranking methods based on a novel representation theory for the permutation model.
Applying this framework to collaborative ranking, we derive asymptotic statistical rates as the number of users and items grow together.
We conclude by demonstrating that our SQL-Rank method often outperforms current state-of-the-art algorithms for implicit feedback such as Weighted-MF and BPR and achieve favorable results when compared to explicit feedback algorithms such as matrix factorization and collaborative ranking. 
\end{abstract}

\section{Introduction}
\label{intro}

We study a novel approach to collaborative ranking---the personalized ranking of items for users based on their observed preferences---through the use of listwise losses, which are dependent only on the observed rankings of items by users.
We propose the SQL-Rank algorithm, which can handle ties and missingness, incorporate both explicit ratings and more implicit feedback, provides personalized rankings, and is based on the relative rankings of items.
To better understand the proposed contributions, let us begin with a brief history of the topic.

\subsection{A brief history of collaborative ranking}

Recommendation systems, found in many modern web applications, movie streaming services, and social media, rank new items for users and are judged based on user engagement (implicit feedback) and ratings (explicit feedback) of the recommended items.
A high-quality recommendation system must understand the popularity of an item and infer a user's specific preferences with limited data. 
Collaborative filtering, introduced in \cite{hill1995recommending}, refers to the use of an entire community's preferences to better predict the preferences of an individual (see \cite{schafer2007collaborative} for an overview).
In systems where users provide ratings of items, collaborative filtering can be approached as a point-wise prediction task, in which we attempt to predict the unobserved ratings \cite{pan2017transfer}.
Low rank methods, in which the rating distribution is parametrized by a low rank matrix (meaning that there are a few latent factors) provides a powerful framework for estimating ratings \cite{mnih2008probabilistic, koren2008factorization}.
There are several issues with this approach.
One issue is that the feedback may not be representative of the unobserved entries due to a sampling bias, an effect that is prevalent when the items are only `liked' or the feedback is implicit because it is inferred from user engagement.
Augmenting techniques like weighting were introduced to the matrix factorization objective to overcome this problem \cite{hsieh2015pu, hu2008collaborative}. Many other techniques are also introduced \cite{kabbur2013fism, wang2017irgan, wu2016collaborative}.
Another methodology worth noting is the CofiRank algorithm of \cite{weimer2008cofi} which minimizes a convex surrogate of the normalized discounted cumulative gain (NDCG).
The pointwise framework has other flaws, chief among them is that in recommendation systems we are not interested in predicting ratings or engagement, but rather we must rank the items. 

Ranking is an inherently relative exercise.
Because users have different standards for ratings, it is often desirable for ranking algorithms to rely only on relative rankings and not absolute ratings.
A ranking loss is one that only considers a user's relative preferences between items, and ignores the absolute value of the ratings entirely, thus deviating from the pointwise framework.
Ranking losses can be characterized as pairwise and listwise.
A pairwise method decomposes the objective into pairs of items $j,k$ for a user $i$, and effectively asks `did we successfully predict the comparison between $j$ and $k$ for user $i$?'.
The comparison is a binary response---user $i$ liked $j$ more than or less than $k$---with possible missing values in the event of ties or unobserved preferences.
Because the pairwise model has cast the problem in the classification framework, then tools like support vector machines were used to learn rankings; \cite{joachims2002optimizing} introduces rankSVM and efficient solvers can be found in \cite{chapelle2010efficient}.
Much of the existing literature focuses on learning a single ranking for all users, which we will call simple ranking \cite{freund2003efficient,agarwal2006ranking, pahikkala2009efficient}.
This work will focus on the personalized ranking setting, in which the ranking is dependent on the user.

Pairwise methods for personalized ranking have seen great advances in recent years, with the AltSVM algorithm of \cite{park2015preference}, Bayesian personalized ranking (BPR) of \cite{rendle2009bpr}, and the near linear-time algorithm of \cite{wu2017large}.
Nevertheless, pairwise algorithms implicitly assume that the item comparisons are independent, because the objective can be decomposed where each comparison has equal weight.
Listwise losses instead assign a loss, via a generative model, to the entire observed ranking, which can be thought of as a permutation of the $m$ items, instead of each comparison independently.
The listwise permutation model, introduced in \cite{cao2007learning}, can be thought of as a weighted urn model, where items correspond to balls in an urn and they are sequentially plucked from the urn with probability proportional to $\phi(X_{ij})$ where $X_{ij}$ is the latent score for user $i$ and item $j$ and $\phi$ is some non-negative function.
They proposed to learn rankings by optimizing a cross entropy between the probability of $k$ items being at the top of the ranking and the observed ranking, which they combine with a neural network, resulting in the ListNet algorithm. 
\cite{shi2010list} applies this idea to collaborative ranking, but uses only the top-1 probability because of the computational complexity of using top-k in this setting.  
This was extended in \cite{huang2015listwise} to incorporate neighborhood information.
\cite{xia2008listwise} instead proposes a maximum likelihood framework that uses the permutation probability directly, which enjoyed some empirical success.

Very little is understood about the theoretical performance of listwise methods.  
\cite{cao2007learning} demonstrates that the listwise loss has some basic desirable properties such as monotonicity, i.e.~increasing the score of an item will tend to make it more highly ranked.
\cite{lan2009generalization} studies the generalizability of several listwise losses, using the local Rademacher complexity, and found that the excess risk could be bounded by a \smash{$1/\sqrt n$} term (recall, $n$ is the number of users).
Two main issues with this work are that no dependence on the number of items is given---it seems these results do not hold when $m$ is increasing---and the scores are not personalized to specific users, meaning that they assume that each user is an independent and identically distributed observation.
A simple open problem is: can we consistently learn preferences from a single user's data if we are given item features and we assume a simple parametric model?  ($n = 1, m\rightarrow \infty$.)

\subsection{Contributions of this work}

We can summarize the shortcomings of the existing work: current listwise methods for collaborative ranking rely on the top-$1$ loss, algorithms involving the full permutation probability are computationally expensive, little is known about the theoretical performance of listwise methods, and few frameworks are flexible enough to handle explicit and implicit data with ties and missingness.
This paper addresses each of these in turn by proposing and analyzing the SQL-rank algorithm.

\begin{compactitem}
    \item We propose the SQL-Rank method, which is motivated by the permutation probability, and has advantages over the previous listwise method using cross entropy loss. 
    \item We provide an $O(\text{iter} \cdot (|\Omega| r))$ linear algorithm based on stochastic gradient descent, where $\Omega$ is the set of observed ratings and $r$ is the rank.
    \item The methodology can incorporate both implicit and explicit feedback, and can gracefully handle ties and missing data.
    \item We provide a theoretical framework for analyzing listwise methods, and apply this to the simple ranking and personalized ranking settings, highlighting the dependence on the number of users and items.
\end{compactitem}

\section{Methodology}
\label{method}
\subsection{Permutation probability}
The permutation probability, \cite{cao2007learning}, is a generative model for the ranking parametrized by latent scores. 
First assume there exists a ranking function that assigns scores to all the items. Let's say we have $m$ items, then the scores assigned can be represented as a vector $s = (s_1, s_2, ..., s_m)$. Denote a particular permutation (or ordering) of the $m$ items as $\pi$, which is a random variable and takes values from the set of all possible permutations $S_m$ (the symmetric group on $m$ elements). $\pi_1$ denotes the index of highest ranked item and $\pi_m$ is the lowest ranked.
The probability of obtaining $\pi$ is defined to be
\begin{equation}
\label{eq:permprob}
    P_{s}(\pi) := \prod_{j=1}^m \frac{\phi(s_{\pi_j})}{\sum_{l=j}^{m} \phi(s_{\pi_l})},
\end{equation}
where $\phi(.)$ is an increasing and strictly positive function. 
An interpretation of this model is that each item is drawn without replacement with probability proportional to $\phi(s_i)$ for item $i$ in each step.
One can easily show that $P_{s}(\pi)$ is a valid probability distribution, i.e.~$\sum_{\pi \in S_m} P_{s}(\pi) = 1, P_{s}(\pi) > 0, \forall \pi$. Furthermore, this definition of permutation probability enjoys several favorable properties (see \cite{cao2007learning}). 
For any permutation $\pi$ if you swap two elements ranked at $i<j$ generating the permutation $\pi'$ ($\pi'_i = \pi_j$, $\pi'_j = \pi_i$, $\pi_k = \pi'_k, k\ne i,j$), if $s_{\pi_i} > s_{\pi_j}$ then $P_s(\pi) > P_s(\pi')$.
Also, if permutation $\pi$ satisfies $s_{\pi_i} > s_{\pi_{i+1}}$, $\forall i$, then we have $\pi = \arg\max_{\pi' \in S_m} P_{s}(\pi')$.
Both of these properties can be summarized: larger scores will tend to be ranked more highly than lower scores.
These properties are required for the negative log-likelihood to be considered sound for ranking \cite{xia2008listwise}.


In recommendation systems, the top ranked items can be more impactful for the performance.
In order to focus on the top $k$ ranked items, we can compute the partial-ranking marginal probability,
\begin{equation}
\label{eq:topk_prob}
    P^{(k,\bar{m})}_{s}(\pi) = \prod_{j=1}^{\min\{k, \bar{m}\}} \frac{\phi(s_{\pi_j})}{\sum_{l=j}^{\bar{m}} \phi(s_{\pi_l})}.
\end{equation}
It is a common occurrence that only a proportion of the $m$ items are ranked, and in that case we will allow $\bar{m} \le m$ to be the number of observed rankings (we assume that $\pi_1,\ldots,\pi_{\bar{m}}$ are the complete list of ranked items).
When $k=1$, the first summation vanishes and top-$1$ probability can be calculated straightforwardly, which is why $k=1$ is widely used in previous listwise approaches for collaborative ranking. 
Counter-intuitively, we demonstrate that using a larger $k$ tends to improve the ranking performance.

We see that computing the likelihood loss is linear in the number of ranked items, which is in contrast to the cross-entropy loss used in \cite{cao2007learning}, which takes exponential time in $k$.
The cross-entropy loss is also not sound, i.e.~it can rank worse scoring permutations more highly, but the 
negative log-likelihood is sound.
We will discuss how we can deal with ties in the following subsection, namely, when the ranking is derived from ratings and multiple items receive the same rating, then there is ambiguity as to the order of the tied items.
This is a common occurrence when the data is implicit, namely the output is whether the user engaged with the item or not, yet did not provide explicit feedback.
Because the output is binary, the cross-entropy loss (which is based on top-$k$ probability with $k$ very small) will perform very poorly because there will be many ties for the top ranked items.
To this end, we propose a collaborative ranking algorithm using the listwise likelihood that can accommodate ties and missingness, which we call Stochastic Queuing Listwise Ranking, or SQL-Rank.

\begin{figure}[ht]
\vspace{-0.1in}
\begin{center}
\centerline{\includegraphics[width=\columnwidth]{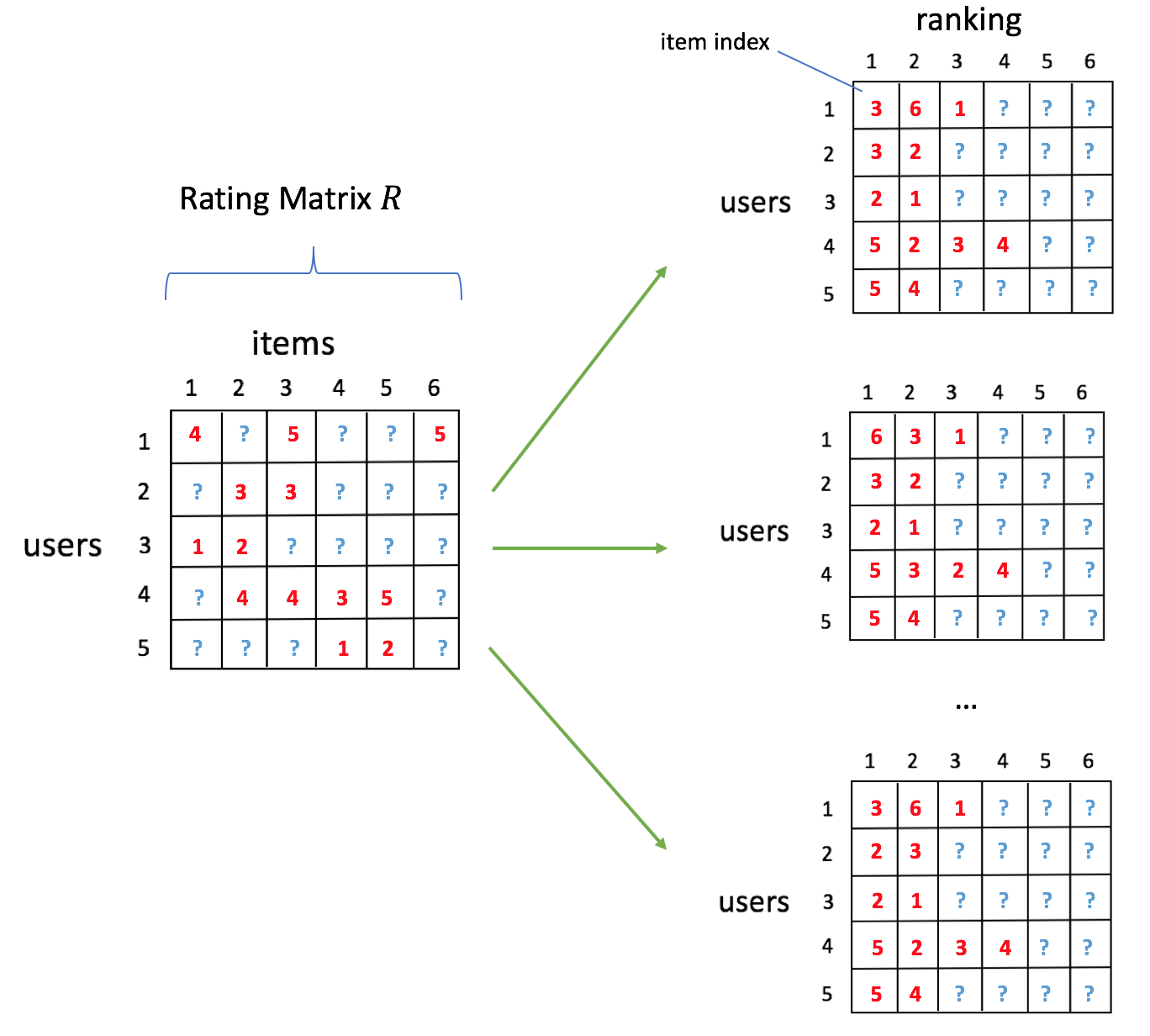}}
\caption{Demonstration of Stochastic Queuing Process---the rating matrix $R$ (left) generates multiple possible rankings $\Pi$'s (right), $\Pi \in \mathcal S(R,\Omega)$ by breaking ties randomly.}
\label{demo}
\end{center}
\vskip -0.3in
\end{figure}

\subsection{Deriving objective function for SQL-Rank}
The goal of collaborative ranking is to predict a personalized score $X_{ij}$ that reflects the preference level of user $i$ towards item $j$, where $1 \leq i \leq n$ and $1 \leq j \leq m$. It is reasonable to assume the matrix $X \in \R^{n \times m}$ to be low rank because there are only a small number of latent factors contributing to users' preferences. The input data is given in the form of ``user $i$ 
gives item $j$ a relevance score $R_{ij}$''. 
Note that for simplicity we assume all the users have the same number $\bar{m}$ of ratings, but this can be easily generalized
to the non-uniform case 
by replacing $\bar{m}$ with $m_i$ (number of ratings for user $i$).

With our scores $X$ and our ratings $R$, we can specify our collaborative ranking model using the permutation probability \eqref{eq:topk_prob}.
Let $\Pi_i$ be a ranking permutation of items for user $i$ (extracted from $R$), we can stack $\Pi_1, \dots \Pi_n$, row by row, to get the permutation matrix $\Pi \in \R^{n\times m}$. 
Assuming users are independent with each other, the probability of observing a particular $\Pi$ given the scoring matrix $X$
can be written as
\begin{equation}
    P_{X}^{(k, \bar{m})}(\Pi) = \prod_{i=1}^{n} P^{(k,\bar{m})}_{X_i}(\Pi_i).
\end{equation}
 We will assume that $\log \phi(x) = 1/(1 + \exp(-x))$ is the sigmoid function.
This has the advantage of bounding the resulting weights, \smash{$\phi(X_{ij})$}, and maintaining their positivity without adding additional constraints.

Typical rating data will contain many {\em ties} within each row.
In such cases, the permutation $\Pi$ is no longer unique and there is a set of permutations
that coincides with rating because with any candidate $\Pi$ we can arbitrarily shuffle the ordering
of items with the same relevance scores to generate a new candidate matrix $\Pi'$ which is still valid (see Figure~\ref{demo}). We denote the set of valid permutations as  $\mathcal S (R, \Omega)$, where $\Omega$ is the set of all pairs $(i,j)$ such that $R_{i,j}$ is observed.
We call this shuffling process the {\it Stochastic Queuing Process}, since one can imagine that by permuting ties we are stochastically queuing new $\Pi$'s for future use in the algorithm.


The probability of observing $R$ therefore should be defined as
    $P_{X}^{(k, \bar{m})}(R) =  \sum_{\Pi \in \mathcal S (R, \Omega)}P_X(\Pi)$. 
To learn the scoring matrix $X$, we can naturally solve the following maximum likelihood estimator with low-rank constraint: 
\begin{equation}
   \min_{X \in \mathcal X} - \log  \sum_{\Pi \in \mathcal S (R, \Omega)} P_X^{(k, \bar{m})}(\Pi), 
    \label{eq:aaa}
\end{equation}
 where $\mathcal{X}$ is the structural constraint of the scoring matrix. 
 To enforce low-rankness, we use the nuclear norm regularization $\mathcal{X}=\{X: \|X\|_*\leq r\}$. 

Eq~\eqref{eq:aaa} is hard to optimize since there is a summation inside the $\log$. 
But by Jensen's inequality and convexity of $-\log$ function, we can move the summation outside $\log$ and obtain an upper bound of the original negative log-likelihood, leading to the following optimization problem: 
\begin{equation}
\min_{X \in \mathcal X} -\sum_{\Pi \in \mathcal S (R, \Omega)} \log P_X^{(k, \bar{m})}(\Pi)
    \label{eq:convex}
\end{equation}
This upper bound is much easier to optimize and can be solved using Stochastic Gradient Descent (SGD). 

Next we discuss how to apply our model for explicit and implicit feedback settings. In the explicit feedback setting, it is assumed that the matrix $R$ is partially observed
and the observed entries are explicit ratings in a range (e.g., $1$ to $5$). 
We will show in the experiments that $k=\bar{m}$ (using the full list) leads to the best results. 
\cite{huang2015listwise} also observed that increasing $k$ is useful for their cross-entropy loss, but they were not able to increase $k$ since their model
has time complexity exponential to $k$. 

In the implicit feedback setting each element of $R_{ij}$ is either $1$ or $0$, where $1$ means positive actions (e.g., click or like) and $0$ means no action is observed.
Directly solving~\eqref{eq:convex} will be expensive since $\bar{m}=m$ and the computation
will involve all the $mn$ elements at each iteration. Moreover, the $0$'s in the matrix 
could mean either a lower relevance score or missing, thus should contribute less to the objective function. 
Therefore, we adopt the idea of negative sampling \cite{mikolov2013distributed} in our list-wise formulation. 
For each user (row of $R$), assume there are $\tilde{m}$ $1$'s, we then 
sample $\rho \tilde{m}$ unobserved entries uniformly from the same row and append to the back of the list. This then becomes the problem with $\bar{m}=(1+\rho) \tilde{m}$ and then we use the same algorithm
in explicit feedback setting to conduct updates. 
We then repeat the sampling process at the end of each iteration, so the update will be based on different set of $0$'s at each time. 

\begin{algorithm}[tb]
   \caption{SQL-Rank: General Framework}
   \label{alg:sqlrank}
\begin{algorithmic}
   \STATE {\bfseries Input:} $\Omega$, $\{R_{ij}: (i, j) \in \Omega \}$, $\lambda\in \R^+$, $ss$, $rate$, $\rho$
   \STATE {\bfseries Output:} $U\in \R^{r \times n}$ and $V \in \R^{r \times m}$
   \STATE Randomly initialize $U, V$ from Gaussian Distribution
   \REPEAT
   \STATE Generate a new permutation matrix $\Pi$ \COMMENT{see alg~\ref{alg:sq}} 
   \STATE Apply gradient update to U while fixing V 
   \STATE Apply gradient update to V while fixing U \COMMENT{see alg~\ref{alg:update}}
   
   \UNTIL{performance for validation set is good}
   \STATE \textbf{return} $U, V$\COMMENT{recover score matrix $X$}
\end{algorithmic}
\end{algorithm}

\begin{algorithm}[tb]
   \caption{Stochastic Queuing Process}
   \label{alg:sq}
\begin{algorithmic}
   \STATE {\bfseries Input:} $\Omega$, $\{R_{ij}: (i, j) \in \Omega \}$, $\rho$  
   \STATE {\bfseries Output:} $\Pi \in \R^{n \times m}$
   \FOR{$i = 1$ {\bfseries to} $n$}
    \STATE Sort items based on observed relevance levels $R_i$
    \STATE Form $\Pi_i$ based on indices of items in the sorted list
    \STATE Shuffle $\Pi_i$ for items within the same relevance level
    \IF{Dataset is implicit feedback}
        \STATE Uniformly sample $\rho \tilde{m}$ items from unobserved items 
        \STATE Append sampled indices to the back of $\Pi_i$
    \ENDIF
   \ENDFOR 
   \STATE Stack $\Pi_i$ as rows to form matrix $\Pi$
   \STATE {\bfseries Return} $\Pi$ \COMMENT{Used later to compute gradient}
\end{algorithmic}
\end{algorithm}

\subsection{Non-convex implementation}
Despite the advantage of the objective function in equation~\eqref{eq:convex} being convex, it is still not feasible for large-scale problems since the scoring matrix $X\in \R^{n\times m}$ leads to high computational and memory cost. We follow a common trick to transform~\eqref{eq:convex} to the non-convex form by replacing $X=U^TV$:
with  $U \in \R^{r \times n}, V\in \R^{r \times m}$ so that the objective is,
\begin{align*}
    \! \sum_{\Pi \in \mathcal S (R, \Omega)}\!\! \underbrace{-\sum_{i=1}^{n} \sum_{j=1}^{\bar{m}}   \log  \frac{\phi(u_i^T v_{\Pi_{ij}})}{\sum_{l=j}^{\bar{m}} \phi(u_i^T v_{\Pi_{il}})}}_{f(U, V)} 
   \! +\! \frac{\lambda}{2} (\|U\|_F^2 \!+\! \|V\|_F^2), 
   \label{eq:nonconvex}
\end{align*}
where $u_i, v_j$ are columns of $U, V$ respectively. We apply stochastic gradient descent to solve this problem. At each step, we choose a permutation matrix $\Pi\in\mathcal S (R, \Omega) $ 
using the stochastic queuing process (Algorithm \ref{alg:sq}) and then update $U, V$ by $\nabla f(U, V)$. For example, the gradient with respect to $V$ is ($g = \log \phi$ is the sigmoid function),
\begin{align*}
\frac{\partial f}{\partial v_j} &= \sum_{i\in \Omega_j} \sum_{t=1}^{\bar{m}}\bigg\{ -g'(u_i^T v_t) u_i \\
& + \frac{\mathbbm{1}(\text{rank}_i(j)\geq t) \phi(u_i^T v_j)}{\sum_{l=t}^{\bar{m}}\phi(u_i^T v_{\Pi_{il}})} g'(u_i^T v_j)u_i\bigg\}
\end{align*}
where $\Omega_j$ denotes the set of users that have rated the item $j$ and $\text{rank}_i(j)$ is a function gives the rank of the item $j$ for that user $i$. Because $g$ is the sigmoid function, $g' = g \cdot (1 - g)$. 
The gradient with respect to $U$ can be derived similarly. 

As one can see, a naive way to compute the gradient of $f$ requires  $O(n \bar{m}^2 r)$ time, which is very slow even for one iteration. However, 
we show in Algorithm~\ref{alg:gradv} (in the appendix) that there is a smart way to re-arranging the computation so that $\nabla_V f(U, V)$ can be computed in $O(n \bar{m} r)$
time, which makes our SQL-Rank a linear-time algorithm (with the same per-iteration complexity as classical matrix factorization). 


\section{Theory}
\label{theory}

Throughout this section, we will establish a theoretical framework for understanding listwise ranking algorithms.
We do not consider ties and missing data and reserve this extension of the theory developed here for future work.
These tools can be employed to analyze any problem of the constrained form 
\begin{equation}
\label{eq:constrainedform}
\hat X := \arg\min - \log P_X(\Pi) \textrm{ such that } X \in \mathcal X.
\end{equation}
We will consider two main settings of listwise ranking, the simple ranking setting where for each $X \in \mathcal X$, 
\begin{equation}
\label{eq:simplepara}
X_{ij} = \beta^\top z_{ij}, \beta \in \mathbb R^{s}, \| \beta \| \le c_b,
\end{equation}
where the feature vectors $z_{ij} \in \mathbb R^{s}$ are known, and the personalized setting, 
\begin{equation}
\label{eq:persmodel}
    X_{ij} = u_i^\top v_j, \\ u_i, v_j \in \mathbb R^{r}, \| U \|_F \le c_u, \| V \|_F \le c_v.
\end{equation}
The simple ranking setting, among other listwise programs was considered in \cite{lan2009generalization}, and it was determined that the excess risk is bounded by a \smash{$1 / \sqrt n$} term.
Critically, these results assumed that the number of items $m$ is bounded, severely limiting their relevance to realistic recommendation systems.
It seems that we should be able to learn something about a user's preferences by having them rank more items, yet the existing theory does not reflect this.

The main engine of our theoretical analysis is a generative mechanism for listwise ranking, which demonstrates that the permutation probability model, \eqref{eq:permprob}, is also the probability of a row-wise ordering of an exponential ensemble matrix.
We demonstrate that the excess risk in the parametric setting scales like $\sqrt{m} \ln m / n$, achieving parametric rates in $n$ and sub-linear excess risk in $m$ when the feature dimension $s$ is fixed.
In the personalized setting, \eqref{eq:persmodel}, we bound the excess risk by $\sqrt{m/n} \ln m$ when the rank $r$ is fixed, which matches comparable results for matrix factorization up to log factors.

\subsection{Generative mechanism for listwise ranking}

We give an alternative generative mechanism which will prove useful for understanding the listwise ranking objective.
\begin{theorem}
\label{thm:gen}
Consider a matrix, $Y$, with independent entries, $Y_{ij}$ that are drawn from an exponential distribution with rate $\phi(X_{ij})$.
Let $\Pi_i$ be the ordering of the entries of $Y_i$ from smallest to largest, then the probability of $\Pi_i | X_i$ is exactly $P_{X_i}(\Pi_i)$.
\end{theorem}

The proof is in the appendix. A key aspect of this generative mechanism is that the listwise likelihood can be written as a function of the exponential ensemble.
This allows us to establish concentration of measure results for the listwise loss via bounded differences.

\subsection{Statistical guarantees}

As a first step to controlling the excess risk, we establish a basic inequality.
This bounds the excess risk by an empirical process term, which is a random function of $\hat X$ and for a fixed $\hat X$ it has mean zero.
The excess risk (the difference in expected loss between the estimate and the truth) can also be written as the KL divergence between the estimated model and the true model.

\begin{lemma}
\label{lem:basic}
Consider the minimizer, $\hat X$, to the constrained optimization, \eqref{eq:constrainedform}.  Suppose that there exists a $X^\star \in \mathcal X$ such that $\Pi_i \sim P_{X^\star_i}$ independently for all \smash{$i=1,\ldots,n$}.  
The KL-divergence between the estimate and the truth is bounded
\begin{align*}
&D(X^\star,\hat X) := \frac 1n \sum_{i=1}^n \mathbb E \log \frac{P_{X_i^\star}(\Pi_i)}{P_{\hat X_i}(\Pi_i)} \tag{basic} \label{eq:basic} \\
&\le - \frac 1n \sum_{i=1}^n \left( \log \frac{P_{X_i^\star}(\Pi_i)}{P_{\hat X_i}(\Pi_i)} - \mathbb E \log \frac{P_{X_i^\star}(\Pi_i)}{P_{\hat X_i}(\Pi_i)} \right).
\end{align*}
\end{lemma}

Because the RHS of \eqref{eq:basic}, the empirical process term, has mean zero and is a function of the random permutation, we can use Theorem \ref{thm:gen} to bound it with high probability for a fixed $\hat X$.
Because $\hat X$ is random, we need to control the empirical process term uniformly over the selection of $\hat X \in \mathcal X$.
To this end, we employ Dudley's chaining, which gives us the following theorem (see the Supplement for the complete proof).

\begin{theorem}
\label{thm:chaining}
Assume the conditions of Lemma \ref{lem:basic}.
Define the matrix norm, for the $n \times m$ matrix $Z$,
$$
\| Z \|_{\infty, 2} := \sqrt{ \sum_{i=1}^n \| Z_i \|_\infty^2} 
$$
and define $\mathcal Z = \{ \log \phi(X) : X \in \mathcal X\}$ where $\log \phi$ is applied elementwise.
Also, let $\mathcal N(\epsilon, \mathcal Z, \| . \|_{\infty,2})$ be the $\epsilon$-covering number of $\mathcal Z$ in the $\infty,2$ norm (the fewest number of $\epsilon$ radius balls in the $\infty,2$ norm required to cover $\mathcal Z$). Then, if $\sup_{Z \in \mathcal Z} \| Z \|_\infty \le C$ (where $\| . \|_\infty$ is the elementwise absolute maximum), then
\begin{equation*}
D(X^\star,\hat X) = O_{\mathbb P} \left( \frac{\sqrt m \ln(m)}{n} \cdot g(\mathcal Z) \right),
\end{equation*}
where 
\begin{equation*}
g(\mathcal Z) := \int_0^\infty \sqrt{\ln \mathcal N (u,\mathcal Z, \|. \|_{\infty,2})} {\rm d}u,
\end{equation*}
and $C$ is bounded by a constant in $n,m$.
\end{theorem}

Theorem \ref{thm:chaining} bounds the KL-divergence by the geometric quantity $g(\mathcal Z)$.
For the derived corollaries, we will assume that $\log \phi$ is 1-Lipschitz, which is true when $\log \phi$ is the sigmoid function.
The results do not significantly change by increasing the Lipschitz constant, so for simplicity of presentation we set it to be 1.

\begin{corollary}
\label{cor:simple}
Assume the conditions to Lemma \ref{lem:basic}, the simple ranking setting \eqref{eq:simplepara}, that $\log \phi$ is 1-Lipschitz, and $\|Z_{ij}\|_2$ is bounded uniformly, then
$$
D(X^\star,\hat X) = O_{\mathbb P} \left( \frac{\sqrt{sm}}{n} \ln m\right).
$$
\end{corollary}

Notably when $n=1$ this bound is on the order of $\sqrt m \ln m$.
In the event that $P_{X^\star}$ is concentrated primarily on a single permutation for this user, and we resort to random guessing (i.e.~$\hat X_{1j} = 0$) then the KL divergence will be close to $\ln m! \approx m \ln m$.  
So, a reduction of the KL-divergence from order $m \ln m$ to $\sqrt m \ln m$ is a large improvement, and the above result should be understood to mean that we can achieve consistency even when $n=1$ (where consistency is measured relative to random guessing).

\begin{corollary}
\label{cor:person}
Assume the conditions to Lemma \ref{lem:basic}, the personalized ranking setting, \eqref{eq:persmodel}, and that $\log \phi$ is 1-Lipschitz,
$$
D(X^\star,\hat X) = O_{\mathbb P} \left( \sqrt{\frac{rm}{n}}\ln m\right).
$$
\end{corollary}

Notably, even in the personalized setting, where each user has their own preferences, we can achieve $1/\sqrt n$ rates for fixed $m,r$.
Throughout these results the $O_{\mathbb P}$ notation only hides the constants $c_b, c_u, c_v$, and any dependence on $s,r,m,n$ is explicitly given.
While Theorem \ref{thm:chaining} gives us a generic result that can be applied to a large range of constraint sets, we believe that the parametric simple ranking and the low-rank personalized setting are the two most important listwise ranking problems.

\section{Experiments}
\label{exp}
In this section, we compare our proposed algorithm (SQL-Rank) with other state-of-the-art algorithms on real world datasets. 
Note that our algorithm works for both implicit feedback and explicit feedback settings. In the implicit feedback setting, all the ratings are $0$ or $1$; 
in the explicit feedback setting, explicit ratings (e.g., $1$ to $5$) are given but only to a subset of user-item pairs. 
Since many real world recommendation systems follow the implicit feedback setting (e.g., purchases, clicks, or checkins), we will first compare SQL-Rank on implicit feedback datasets and show it outperforms state-of-the-art algorithms. Then we will verify that our algorithm also performs well on explicit feedback problems. 
All experiments are conducted on a server with an Intel Xeon E5-2640 2.40GHz CPU and 64G RAM.

\subsection{Implicit Feedback}
In the implicit feedback setting we compare the following methods:
\begin{compactitem}
    \item SQL-Rank: our proposed algorithm implemented in Julia \footnote{\url{https://github.com/wuliwei9278/SQL-Rank}}. 
    \item Weighted-MF: the weighted matrix factorization algorithm by putting different weights on $0$ and $1$'s \cite{hu2008collaborative,hsieh2015pu}.
    \item BPR: the Bayesian personalized ranking method motivated by MLE \cite{rendle2009bpr}. For both Weighted-MF and BPR, we use the C++ code by Quora \footnote{\url{https://github.com/quora/qmf}}. 
\end{compactitem}
Note that other collaborative ranking methods such as Pirmal-CR++~\cite{wu2017large} and List-MF~\cite{shi2010list}
do not work for implicit feedback data, and we will compare with them later in the explicit feedback experiments. 
For the performance metric, we use precision@$k$ for $k = 1, 5, 10$ defined by 
\begin{equation}
    \text{precision}@k = \frac{\sum_{i=1}^{n}|\{1 \leq l \leq k: R_{i\Pi_{il}} = 1\}|}{n \cdot k},
\end{equation} where $R$ is the rating matrix and $\Pi_{il}$ gives the index of the $l$-th ranked item for user $i$ among all the items not rated by user $i$ in the training set. 

We use rank $r = 100$ and tune regularization parameters for all three algorithms using a random sampled validation set. For Weighted-MF, we also tune the confidence weights on unobserved data. For BPR and SQL-Rank, we fix the ratio of subsampled unobserved $0$'s versus observed $1$'s to be $3:1$, which gives the best performance for both BPR and SQL-rank in practice. 

We experiment on the following four datasets. 
Note that the original data of Movielens1m, Amazon and Yahoo-music are
ratings from $1$ to $5$, so we follow the procedure in \cite{rendle2009bpr,yu2017selection} to preprocess the data. 
We transform ratings of $4, 5$ into $1$'s and the rest entries (with rating $1, 2, 3$ and unknown) as $0$'s.
Also, we remove users with very few $1$'s in the corresponding row to make sure there are enough $1$'s for both training and testing. 
For Amazon, Yahoo-music and Foursquare, we discard users with less than $20$ ratings and randomly select $10$ $1$'s as training and use the rest as testing. 
Movielens1m has more ratings than others, so we keep users with more than $60$ ratings, and randomly sample 50 of them as training. 
\begin{compactitem}
\item Movielens1m: a popular movie recommendation data with $6,040$ users and $3,952$ items.
\item Amazon: the Amazon purchase rating data for musical instruments \footnote{\url{http://jmcauley.ucsd.edu/data/amazon/}} with $339,232$ users and $83,047$ items. 
\item Yahoo-music: the Yahoo music rating data set \footnote{\url{https://webscope.sandbox.yahoo.com/catalog.php?datatype=r&did=3}} which contains $15,400$ users and $1,000$ items. 
\item Foursquare: 
a location check-in data\footnote{\url{https://sites.google.com/site/yangdingqi/home/foursquare-dataset}}. The data set contains $3,112$ users and $3,298$ venues with $27,149$ check-ins. The data set is already in the form of ``0/1'' so we do not need to do any transformation. 
\end{compactitem}

The experimental results are shown in Table~\ref{implicit-combined}. We find that SQL-Rank outperforms both Weighted-MF and BPR in most cases.

\begin{table}[t]
\caption{Comparing implicit feedback methods on various datasets.}
\label{implicit-combined}
\begin{center}
\begin{small}
\begin{sc}
\resizebox{0.5\textwidth}{!}{
\begin{tabular}{lccccr}
\toprule
Dataset & Method &  P@1 & P@5 & P@10 \\
\midrule
\multirow{3}{*}{Movielens1m} & SQL-Rank   &  {\bfseries 0.73685}  & {\bfseries 0.67167}   & 0.61833 \\
& Weighted-MF & 0.54686 & 0.49423 &  0.46123  \\
& BPR   & 0.69951 & 0.65608 &  {\bfseries 0.62494}  \\
\midrule
\multirow{3}{*}{Amazon} & SQL-Rank   &  0.04255  & {\bfseries 0.02978}   & {\bfseries 0.02158} \\
& Weighted-MF & 0.03647 & 0.02492 &  0.01914  \\
& BPR   & {\bfseries 0.04863} & 0.01762 &   0.01306  \\
\midrule
\multirow{3}{*}{Yahoo music} & SQL-Rank   &  {\bfseries 0.45512}  & {\bfseries 0.36137}   & {\bfseries 0.30689} \\
& Weighted-MF &  0.39075 & 0.31024 & 0.27008  \\
& BPR   & 0.37624 & 0.32184 &   0.28105  \\
\midrule
\multirow{3}{*}{Foursquare} & SQL-Rank   &  {\bfseries 0.05825}  & {\bfseries 0.01941}   & {\bfseries 0.01699} \\
& Weighted-MF & 0.02184 & 0.01553 &  0.01407  \\
& BPR   & 0.03398 & 0.01796 &   0.01359  \\
\bottomrule
\end{tabular}
}
\end{sc}
\end{small}
\end{center}
\vskip -0.1in
\end{table}

\subsection{Explicit Feedback}
Next we compare the following methods in the explicit feedback setting: 
\begin{compactitem}
    \item SQL-Rank: our proposed algorithm implemented in Julia. Note that in the explicit feedback setting our algorithm
    only considers pairs with explicit ratings. 
    \item List-MF: the listwise algorithm using the cross entropy loss between observed rating and top $1$ probability \cite{shi2010list}. We use the C++ implementation on github\footnote{\url{https://github.com/gpoesia/listrankmf}}.
    \item MF: the classical matrix factorization algorithm in \cite{koren2008factorization} utilizing a pointwise loss solved by SGD. We implemented SGD in Julia. 
    \item Primal-CR++: the recently proposed pairwise algorithm in \cite{wu2017large}. We use the Julia implementation released by the authors\footnote{\url{https://github.com/wuliwei9278/ml-1m}}.
\end{compactitem}

Experiments are conducted on Movielens1m and Yahoo-music datasets. We perform the same procedure as in implicit feedback setting except that we do not need to mask the ratings into ``0/1''. 

We measure the performance in the following two ways:
\begin{compactitem}
    \item NDCG$@k$: defined as:
        \begin{equation*} 
            \text{NDCG}@k = \frac{1}{n} \sum_{i = 1}^{n} \frac{\text{DCG}@k(i, \Pi_i)}{\text{DCG}@k(i, \Pi_i^*)}, 
        \end{equation*} where $i$ represents $i$-th user and
        \begin{equation*} 
        \text{DCG}@k(i, \Pi_i)= \sum_{l = 1}^{k} \frac{2^{R_{i\Pi_{il}}} - 1}{log_2(l + 1)}. 
        \end{equation*}
        In the DCG definition, $\Pi_{il}$ represents the index of the $l$-th ranked item for user $i$ in test data based on the learned score matrix $X$. $R$ is the rating matrix and $R_{ij}$ is the rating given to item $j$ by user $i$. $\Pi_i^*$ is the ordering provided by the ground truth rating.
    \item Precision@$k$: defined as a fraction of relevant items among the top $k$ recommended items: 
        \begin{equation*}
            \text{precision}@k = \frac{\sum_{i=1}^{n}|\{1 \leq l \leq k: 4 \leq R_{i\Pi_{il}} \leq 5\}|}{n \cdot k},
        \end{equation*}
        here we consider items with ratings assigned as $4$ or $5$ as relevant. $R_{ij}$ follows the same definitions above but unlike before $\Pi_{il}$ gives the index of the $l$-th ranked item for user $i$ among all the items that are not rated by user $i$ in the training set (including both rated test items and unobserved items). 
\end{compactitem}

As shown in Table~\ref{explicit-combined}, our proposed listwise algorithm SQL-Rank outperforms previous listwise method List-MF in both NDCG@$10$ and precision@$1,5,10$. It verifies the claim that log-likelihood loss outperforms the cross entropy loss if we use it correctly. When listwise algorithm SQL-Rank is compared with pairwise algorithm Primal-CR++, the performances between SQL-Rank and Primal-CR++ are quite similar, slightly lower for NDCG@$10$ but higher for precision@$1,5,10$. Pointwise method MF is doing okay in NDCG but really bad in terms of precision. 
Despite having comparable NDCG, the predicted top $k$ items given by MF are quite different from those given by other algorithms utilizing a ranking loss. The ordered lists based on SQL-Rank, Primal-CR++ and List-MF, on the other hand, share a lot of similarity and only have minor difference in ranking of some items. It is an interesting phenomenon that we think is worth exploring further in the future. 

\begin{table}[t]
\caption{Comparing explicit feedback methods on various datasets.}
\label{explicit-combined}
\begin{center}
\begin{small}
\begin{sc}
\resizebox{0.5\textwidth}{!}{
\begin{tabular}{lccccr}
\toprule
Dataset & Method & NDCG@$10$ & P@1 & P@5 & P@10 \\
\midrule
\multirow{3}{*}{Movielens1m} & SQL-Rank &  0.75076   &  {\bfseries 0.50736}  & {\bfseries 0.43692}   & {\bfseries 0.40248} \\
& List-MF & 0.73307 & 0.45226 &  0.40482   & 0.38958 \\
& Primal-CR++    & {\bfseries 0.76826} & 0.49365 & 0.43098   & 0.39779 \\
& MF    & 0.74661 & 0.00050 & 0.00096   & 0.00134 \\
\midrule
\multirow{3}{*}{Yahoo music} & SQL-Rank &  0.66150   &  {\bfseries 0.14983}  & {\bfseries 0.12144}   & {\bfseries 0.10192} \\
& List-MF & 0.67490 & 0.12646 &  0.11301   & 0.09865 \\
& Primal-CR++    &  0.66420 & 0.14291 & 0.10787   & 0.09104 \\
& MF    & {\bfseries 0.69916} & 0.04944 & 0.03105   & 0.04787 \\

\bottomrule
\end{tabular}
}
\end{sc}
\end{small}
\end{center}
\vskip -0.2in
\end{table}


\subsection{Training speed}

To illustrate the training speed of our algorithm, we plot precision@$1$ versus training time for the Movielen1m dataset and the Foursquare dataset. Figure~\ref{time} and Figure~\ref{time2} (in the appendix) show that our algorithm SQL-Rank is faster than BPR and Weighted-MF. Note that our algorithm is implemented in Julia while BPR and Weighted-MF are highly-optimized C++ codes (usually at least 2 times faster than Julia) released by Quora.  
This speed difference makes sense as our algorithm takes $O(n \bar{m} r)$ time, which is linearly to the observed ratings. In comparison, pair-wise model such as BPR
has $O(n\bar{m}^2)$ pairs, so will take $O(n \bar{m}^2 r)$ time for each epoch. 

 \begin{figure}[ht]
\vskip -0.1in
\begin{center}
\centerline{\includegraphics[width=0.8\columnwidth]{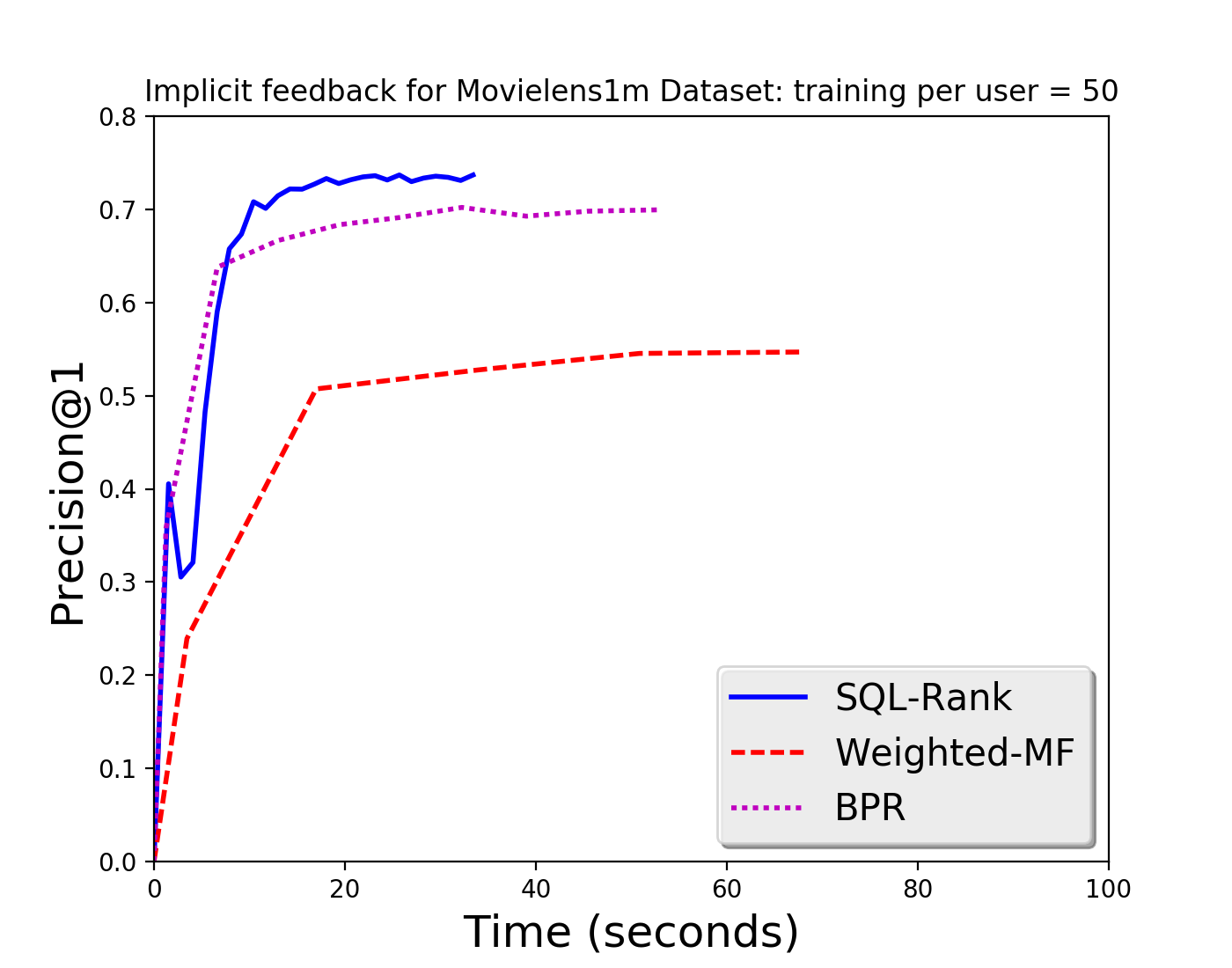}}
\caption{Training time of implicit feedback methods.}
\label{time}
\end{center}
\vskip -0.4in
\end{figure}

\subsection{Effectiveness of Stochastic Queuing (SQ)}

One important innovation in our SQL-Rank algorithm is the Stochastic Queuing (SQ) Process for handling ties. 
To illustrate the effectiveness of the SQ process, we compare our algorithm with and without SQ. Recall that without SQ means we fix a certain permutation matrix $\Pi$ and optimize with respect to it throughout all iterations without generating new $\Pi$, while SQ allows us to update using a new permutation at each time. As shown 
Table~\ref{implicit-sq} and Figure~\ref{sqp} (in the appendix),
the performance gain from SQ in terms of precision is substantial (more than $10\%$) on Movielen1m dataset. It verifies the claim that our way of handling ties and missing data is very effective and improves the ranking results by a lot. 

\begin{table}[t]
\caption{Effectiveness of Stochastic Queuing Process.}
\label{implicit-sq}
\vskip 0.15in
\begin{center}
\begin{small}
\begin{sc}
\resizebox{0.35\textwidth}{!}{
\begin{tabular}{lcccr}
\toprule
Method &  P@1 & P@5 & P@10 \\
\midrule
With SQ   &  {\bfseries 0.73685}  & {\bfseries 0.67167}   & {\bfseries 0.61833} \\
Without SQ & 0.62763 & 0.58420 &  0.55036  \\
\bottomrule
\end{tabular}
}
\end{sc}
\end{small}
\end{center}
\vskip -0.2in
\end{table}

%


\subsection{Effectiveness of using the Full List}

Another benefit of our algorithm is that we are able to minimize top $k$ probability with much larger $k$ and without much overhead. Previous approaches 
\cite{huang2015listwise}
already pointed out increasing $k$ leads to better ranking results, but
their complexity is exponential to $k$ so they were not able to have $k > 1$. 
To show the effectiveness of using permutation probability for full lists rather than using the top $k$ probability for top $k$ partial lists in the likelihood loss, we fix everything else to be the same and only vary $k$ in Equation~\eqref{eq:convex}. We obtain the results in Table~\ref{explicit-ml1m-topk} and Figure~\ref{explicit_topk} (in the appendix). It shows that the larger $k$ we use, the better the results we can get. Therefore, in the final model, we set $k$ to be the maximum number (length of the  observed list.)   


\vspace{-10pt}\begin{table}[th]
\caption{Comparing different $k$ on Movielens1m data set using 50 training data per user.}
\label{explicit-ml1m-topk}
\vskip 0.15in
\begin{center}
\begin{small}
\begin{sc}
\vspace{-15pt}\resizebox{0.5\textwidth}{!}{
\begin{tabular}{lcccr}
\toprule
$k$ & NDCG@$10$ & P@1 & P@5 & P@10 \\
\midrule
5 &  0.64807   &  0.39156  & 0.33591   & 0.29855 \\
10 & 0.67746 & 0.43118 &  0.34220   & 0.33339 \\
25    & 0.74589 & 0.47003 & 0.42874   & 0.39796 \\
50 (full list) & {\bfseries 0.75076} & {\bfseries 0.50736} & {\bfseries 0.43692}  & {\bfseries 0.40248} \\
\bottomrule
\end{tabular}
}
\end{sc}
\end{small}
\end{center}
\vskip -0.2in
\end{table}

\section{Conclusions}
In this paper, we propose a listwise approach for collaborative ranking and provide
an efficient algorithm to solve it. 
Our methodology can incorporate both implicit and explicit feedback, and can gracefully handle ties and missing data. 
In experiments, we demonstrate our algorithm outperforms existing state-of-the art methods in terms of top $k$ recommendation precision. 
We also provide a theoretical framework for analyzing listwise methods highlighting the dependence on the number of users and items. 

\newpage

\section*{Acknowledgements}

JS is partially supported by NSF DMS-1712996. CJH acknowledge the support by NSF IIS-1719097, Google Cloud and Nvidia. \\

Side Note by Liwei Wu: SQL in SQL-Rank is not only the abbreviation for Stochastically Queuing Listwise, but also name initials of Liwei's girlfriend ShuQing Li. Special thanks for her support. 




\bibliography{listwise}
\bibliographystyle{icml2018}

\newpage~\newpage
\appendix
\section{Supplement to ``A Listwise Approach to Collaborative Ranking"}




\subsection{Proofs in Theory section}

\begin{proof}[Proof of Theorem \ref{thm:gen}.]
Notice that $\Pi_{i1}$ is the argument, $k$, that minimizes $Y_{ik}$, and 
$\mathbb P \{ \Pi_{i1} = k \} = \mathbb P\{ Y_{ik} \le \min \{ Y_{ij}\}_{j \ne k} \}$.
Furthermore, $\min \{ Y_{ij}\}_{j \ne k}$ is exponential with rate parameter $\sum_{j\ne k} \phi(X_{ij})$ and is independent of $Y_{ik}$.  Hence,
\begin{align*}
&\mathbb P\left\{ Y_{ik} \le \min \{ Y_{ij}\}_{j \ne k} \right\} \\
&= \int_0^\infty \phi(X_{ik}) e^{-u \phi(X_{ik})} e^{- \sum_{j \ne k} u \phi(X_{ij})} {\rm d}u \\
&= \frac{\phi(X_{ik})}{\sum_j \phi(X_{ij})}.
\end{align*}
Furthermore,
\begin{align*}
&\mathbb P \{ \Pi_i | \Pi_{i1} \} = \mathbb P \{ Y_{\Pi_{i_2}} \le \ldots \le Y_{\Pi_{im}} | Y_{\Pi_{ij}} \ge Y_{\Pi_{i1}}, \forall j\} \\
&= \mathbb P \{ Y_{\Pi_{i2}} - Y_{\Pi_{i1}} \le \ldots \le Y_{\Pi_{im}} - Y_{\Pi_{i1}} | Y_{\Pi_{ij}} \ge Y_{\Pi_{i1}}, \forall j\}.
\end{align*}
By the memorylessness property, we have that the joint distribution of $Y_{\Pi_{i2}} - Y_{\Pi_{i1}}, \ldots, Y_{\Pi_{im}} - Y_{\Pi_{i1}} | Y_{\Pi_{ij}} \ge Y_{\Pi_{i1}}$, $\forall j>1$ is equivalent to the distribution of $Y_{\Pi_{i2}}, \ldots, Y_{\Pi_{im}}$.
Hence, we can apply induction with the previous argument, and the tower property of conditional probability.
\end{proof}

\begin{proof}[Proof of Lemma \ref{lem:basic}.]
By optimality, 
\begin{equation*}
\frac 1n \sum_{i=1}^n - \log P_{\hat X_i}(\Pi_i) \le \frac 1n \sum_{i=1}^n - \log P_{X_i^\star}(\Pi_i).
\end{equation*}
Which is equivalent to 
\begin{equation*}
\frac 1n \sum_{i=1}^n - \log \frac{P_{X_i^\star}(\Pi_i)}{P_{\hat X_i}(\Pi_i)} \ge 0.
\end{equation*}
Thus, we can subtract the expectation,
\begin{align*}
&\frac 1n \sum_{i=1}^n \mathbb E \log \frac{P_{X_i^\star}(\Pi_i)}{P_{\hat X_i}(\Pi_i)} \\
&\le - \frac 1n \sum_{i=1}^n \left( \log \frac{P_{X_i^\star}(\Pi_i)}{P_{\hat X_i}(\Pi_i)} - \mathbb E \log \frac{P_{X_i^\star}(\Pi_i)}{P_{\hat X_i}(\Pi_i)} \right)
\end{align*}
where the expectation $\mathbb E$ is with respect to the draw of $\Pi_i$ conditional on $X$.
\end{proof}

\begin{lemma}
\label{lem:bdd_diffs}
Let $\pi$ be a permutation vector and $x$ be a score vector each of length $m$.
Suppose that $|\log \phi(x_j)| \le C$ for all $j=1,\ldots, m$. 
Define the relative loss function,
$$
L_{x,x'}(\pi) := \log \frac{P_{x}(\pi)}{P_{x'}(\pi)}.
$$
Consider translating an item in position $\ell$ to position $\ell'$ in the permutation $\pi$, thus forming $\pi'$ where $\pi'_{\ell'} = \pi_\ell$.
Specifically,
$\pi'_k = \pi_k$ if $k < \min\{\ell,\ell'\}$ or $k > \max\{\ell,\ell'\}$;
if $\ell < \ell'$ then $\pi'_k = \pi_{k+1}$ for $k=\ell,\ldots,\ell'-1$;
if $\ell' < \ell$ then $\pi'_k = \pi_{k-1}$ for $k=\ell'+1,\ldots,\ell$.
The relative loss function has bounded differences in the sense that
$$
|L_{x,x'}(\pi) - L_{x,x'}(\pi')| \le C_0 \|\log \phi(x) - \log \phi(x') \|_\infty,
$$
where $\phi(x)$ is applied elementwise and $C_0 = 2 + e^{2C} \ln(m+1)$.
\end{lemma}

\begin{proof}
Suppose that $\ell < \ell'$, and define the following shorthand,
$$
\lambda_j = \phi(x_{\pi_j}), \quad \Lambda_j = \sum_{k=j}^m \lambda_k,
$$
and let $\lambda_j', \Lambda_j'$ be similarly defined with $x'$.
Then by replacing the permutation $\pi$ with $\pi'$ causes the $\Lambda_j$ to be replaced with $\Lambda_j - \lambda_j + \lambda_\ell$ for $j=\ell+1,\ldots,\ell'$.
Hence,
\begin{align*}
&\log Q_x(\pi) - \log Q_x(\pi') \\
& = \sum_{j=\ell+1}^{\ell'} \log\left( \Lambda_j -\lambda_j + \lambda_\ell \right) - \log\left( \Lambda_j \right) \\
& = \sum_{j=\ell+1}^{\ell'} \log\left( \Lambda_{j-1} + \lambda_\ell \right) - \log\left( \Lambda_j \right) \\
&= \sum_{j=\ell}^{\ell' - 1} \log\left(1 + \frac{\lambda_\ell}{\Lambda_j}\right) + \log \Lambda_\ell - \log \Lambda_{\ell'}.
\end{align*}
So we can bound the difference,
\begin{align*}
&|L_{x,x'}(\pi) - L_{x,x'}(\pi')| \le \left| \log \frac{\Lambda_\ell}{\Lambda_\ell'} - \log \frac{\Lambda_{\ell'}}{\Lambda_{\ell'}'} \right|\\
&+ \sum_{j=\ell}^{\ell'-1} \left|\log\left(1 + \frac{\lambda_\ell }{\Lambda_j}\right) - \log\left(1 + \frac{\lambda_\ell'}{\Lambda_j'}\right)\right|.
\end{align*}
Suppose that for each $j$, $|\log \lambda_j - \log \lambda_j'| \le \delta$ and that $|\log \lambda_j| \le C$.
Then we have that
\begin{align*}
    \left| \log \frac{\Lambda_\ell}{\Lambda_{\ell}'} \right| \le \max_{j \ge \ell}\left| \log \frac{\lambda_j}{\lambda_j'} \right| \le \delta.
\end{align*}
The same equation can be made for this term with $\ell'$.
Let 
$$
\alpha_j = \max \left\{ \frac{\lambda_\ell}{\Lambda_j}, \frac{\lambda_\ell'}{\Lambda_j'} \right\}.
$$
Then we have 
\begin{align*}
&\left| \log\left(1 + \frac{\lambda_\ell}{\Lambda_j}\right) - \log\left(1 + \frac{\lambda_\ell'}{\Lambda_j'}\right) \right| \\
&\le \left| \log (1 + \alpha_j ) - \log (1 + e^{-\delta} \alpha_j ) \right| \le \left| 1 - e^{-\delta} \right| |\alpha_j|.
\end{align*}
Furthermore, because $\Lambda_j \ge (m-j +1) e^{-C}$ then $|\alpha_j| \le (m-j+1)^{-1} e^{2C}$ and
\begin{align*}
&|L_{x,x'}(\pi) - L_{x,x'}(\pi')| \le 2 \delta \\
&+ \sum_{j=\ell}^{\ell'-1} |1 - e^{-\delta}| \frac{e^{2C}}{m - j + 1} \\
&\le 2 \delta + |1 - e^{-\delta}| e^{2C} H_m\\
&\le \delta (2 + e^{2C} \ln(m+1)).
\end{align*}
In the above equation, $H_m$ is the $m$th harmonic number.
A similar calculation can be made when $\ell' > \ell$.
Setting $\delta = \|\log \phi(x) - \log \phi(x') \|_\infty$ concludes the proof.
\end{proof}

\begin{proof}[Proof of Theorem \ref{thm:chaining}.]
Define the empirical process function to be
$$
\rho_n(x) := \frac 1n \sum_{i=1}^n \left( \log \frac{P_{X_i^\star}(\Pi_i)}{P_{X_i}(\Pi_i)} - \mathbb E \log \frac{P_{X_i^\star}(\Pi_i)}{P_{X_i}(\Pi_i)} \right).
$$
By the listwise representation theorem, $\rho_n(x)$ is a function of $n \times m$ independent exponential random variables.
Moreover, if we were to change the value of a single element $y_{ik}$ then this would result in a change of permutation of the type described in Lemma \ref{lem:bdd_diffs}.
Notice that the bound on the change in the relative loss is $C_0 \|\log \phi(X_i) - \log \phi(X_i') \|_\infty$, where $C_0 = 2 + e^{2C} \ln(m+1)$, and notice that the sum of squares of these bounds are,
\begin{align*}
&\sum_{i,k} C_0^2 \|\log \phi(X_i) - \log \phi(X'_i)\|^2_\infty \\
&= m C_0^2 \sum_{i=1}^n \| \log \phi(X_i) - \log \phi(X'_i) \|_\infty^2 \\
&= m C_0^2 \| Z - Z' \|_{\infty,2}^2,
\end{align*}
where $Z,Z'$ are $\log \phi$ applied elementwise to $X,X'$ respectively.
By Lemma \ref{lem:bdd_diffs} and McDiarmid's inequality,
$$
\mathbb P \{ n(\rho_n(x) - \rho_n(x')) > \epsilon \} \le \exp\left( - \frac{2 \epsilon^2}{m C_0^2 \| Z - Z'\|_{\infty,2}^2} \right).
$$
Hence, the stochastic process $\{ n \rho_n(X): X\in \mathcal X \}$ is a subGaussian field with canonical distance, 
$$
d(X,X') :=  \sqrt m C_0 \| Z - Z'\|_{\infty,2}.
$$
The result follows by Dudley's chaining \cite{talagrand2006generic}.
\end{proof}

\begin{lemma}
\label{lem:lipschitz}
If $\log \phi$ is 1-Lipschitz then we have that 
$g(\mathcal Z) \le g(\mathcal X)$.
\end{lemma}

\begin{proof}
Let $X,X' \in \mathcal X$, and $Z = \log \phi(X), Z' = \log \phi(X')$.
Then
$$
|z_{ij} - z_{ij}'| \le |x_{ij} - x_{ij}'|,
$$
by the Lipschitz property.
Hence, $\| Z - Z' \|_{\infty,2} \le \|X - X'\|_{\infty, 2}$, and so
$$
\mathcal N(u,\mathcal Z, \| . \|_{\infty,2}) \le \mathcal N(u,\mathcal X, \|.\|_{\infty,2}).
$$
\end{proof}

\begin{proof}[Proof of Corollary \ref{cor:simple}.]
Consider two matrices $X,X'$ in the model \eqref{eq:simplepara} then
$$
\| x_i - x_i' \|_\infty = \max_j |\beta^\top z_{ij} - \beta'^\top z_{ij}| \le \|\beta - \beta'\|_2 \|Z_i\|_{2,\infty }.
$$
Let $\zeta = \max_i \|Z_i\|_{2,\infty }$ then 
$$
\| X -X' \|_{\infty,2} \le \zeta \| \beta - \beta'\|_2.
$$
The covering number of $\mathcal X$ is therefore bounded by
$$
\mathcal N(\mathcal X, u, \| . \|_{\infty,2}) \le \mathcal N(B_{c_b},u/ \zeta,\|.\|_2) \le \left(\frac{\zeta C_0 c_b}{u}\right)^s,
$$
for an absolute constant $C_0$, where $B_{c_b}$ is the $\ell_2$ ball of radius $c_b$.
The result follows by Lemma \ref{lem:lipschitz} and Theorem \ref{thm:chaining}.
\end{proof}

\begin{proof}[Proof of Corollary \ref{cor:person}.]
Let $X = U V^\top$ and $X' = U' V'^\top$ such that $U,V,U',V'$ bounded by 1 in Frobenius norm (modifying the Frobenius norm bound does not change the substance of the proof).
Consider
\begin{align*}
&|u_i^\top v_j - u_i'^\top v_j'| \le |u_i^\top v_j - u_i'^\top v_j| + |u_i'^\top v_j - u_i'^\top v_j'| \\
&\le \| u_i - u_i' \| \| v_j \| + \| u_i' \| \| v_j - v_j' \|.
\end{align*}
Maximizing this over the selection of $j$,
\begin{align*}
&\max_j |u_i^\top v_j - u_i'^\top v_j'| \\
&\le \| u_i - u_i' \|_2 \| V \|_{2,\infty} + \| u_i \|_2 \| V - V' \|_{2,\infty}.
\end{align*}
Hence,
\begin{align*}
&\| X - X' \|_{\infty, 2} \\
&\le \| U - U' \|_F \| V \|_{2,\infty} + \| U' \|_F \| V - V' \|_{2,\infty} \\
&\le \| U - U' \|_F + \| V - V' \|_{2,\infty}.
\end{align*}
Consider the vectorization mapping from the $m \times r$ matrix to the $mr$ dimensional vectors.  The Frobenius norm is mapped to the $\ell_2$ norm, and we can consider the $2,\infty$ norm to be, the norm $\| x \|_\rho = \max_{j} \| (x_{jr+1},\ldots,x_{j(r+1)}) \|_2$.
The $\rho$-norm unit ball ($B_\rho$) is just the Cartesian product of the $\ell_2$ norm ball in $K$ dimensions.
The volume of a $d$-dimensional ball, $V_{d}$, is bounded by
$$
C_l \le \frac{V_d}{(e \pi)^{d/2}} d^{\frac d2} \le C_u,
$$
where $C_l < C_u$ are universal constants.
So the volume ratio between the $\ell_2$ norm ball and the $\rho$ norm ball is bounded by
\begin{align*}
&\frac{V(B_2)}{V(B_\rho)} \le C \left( \frac{{r}^{r / 2}}{(e \pi)^{r/2}} \right)^{m} \big/ \left( \frac{{(rm)}^{rm / 2}}{(e \pi)^{rm/2}} \right)\\
&\le C m^{- rm/2},
\end{align*}
where $C = C_u / C_l$.
\begin{align*}
&\mathcal N(\epsilon,B_2,\| . \|_\rho) \le C_r \left( \frac 2\epsilon + 1 \right)^{rm} m^{- rm/2} \\
&\le C \left( \frac{3}{\epsilon \sqrt m} \right)^{rm},
\end{align*}
for $\epsilon \le 1$.  This is also the covering number of the Frobenius norm ball in the $2,\infty$ norm.
Moreover, we know that the covering number of the unit Frobenius norm ball in $n \times K$ matrices ($B_F$) in the Frobenius norm is 
$$
\mathcal N(\epsilon,B_F,\| . \|_F) \le \left( \frac{c}{\epsilon} \right)^{nr},
$$
for some constant $c$.
Consider covering the space $\mathcal X$, by selecting centers $U,V$ from the $\epsilon/2$-coverings of $B_F$ in the $F$-norm and ${2,\infty}$ norm respectively.
By the above norm bound, this produces an $\epsilon$-covering in the $\infty,2$ norm.  
Dudley's entropy bound is thus
\begin{align*}
&\int_0^\infty \sqrt{\log \mathcal N(\epsilon, B_F, \|.\|_F) + \log \mathcal N (\epsilon, B_F, \|. \|_{2,\infty})} {\rm d} \epsilon \\
&\le \int_0^c \sqrt{nr \log (c / \epsilon)} {\rm d}\epsilon + \int_0^{3/\sqrt m} \sqrt{- mr \log (\epsilon \sqrt m / 3)} {\rm d}\epsilon.
\end{align*}
So that
\begin{align*}
&\int_0^c \sqrt{nr \log (c / \epsilon)} {\rm d}\epsilon \\
&\le c' \sqrt{nr}
\end{align*}
for some absolute constant $c'$ and
\begin{align*}
&\int_0^{3/\sqrt m} \sqrt{- mr \log (\epsilon \sqrt m / 3)} {\rm d}\epsilon \le \int_0^3 \sqrt{r \log (u/3)} {\rm d} u \\
&\le c' \sqrt r.
\end{align*}
Hence, for $g(\mathcal X) \le c' \sqrt{nr}$ and we have the result.
\end{proof}

\subsection{Algorithms}
\begin{algorithm}[tb]
  \caption{Compute gradient for $V$ when $U$ fixed}
  \label{alg:gradv}
\begin{algorithmic}
  \STATE {\bfseries Input:} $\Pi$, $U$, $V$, $\lambda, \rho$
  \STATE {\bfseries Output:} $g$ \COMMENT{$g\in \R^{r \times m}$ is the gradient for $f(V)$}
    \STATE $g = \lambda \cdot V$
  \FOR{$i = 1$ {\bfseries to} $n$}
     \STATE Precompute $h_t = u_i^T v_{\Pi_{it}}$ for $1 \leq t \leq \bar{m}$ \COMMENT{For implicit feedback, it should be $(1 + \rho) \cdot \tilde{m}$ instead of $\tilde{m}$, since $\rho \cdot \tilde{m}$ $0$'s are appended to the back}
     \STATE Initialize $total = 0$, $tt = 0$ 
     \FOR{$t = \bar{m}$ {\bfseries to} $1$}
        \STATE $total \mathrel{+}= \exp(h_t)$
        \STATE $tt \mathrel{+}= 1/total$
     \ENDFOR
     \STATE Initialize $c[t] = 0$ for $1 \leq t \leq \bar{m}$ 
     \FOR{$t = \bar{m}$ {\bfseries to} $1$}
        \STATE $c[t] \mathrel{+}= h_t \cdot (1 - h_t) $
        \STATE $c[t] \mathrel{+}= \exp(h_t) \cdot h_t \cdot (1 - h_t) \cdot tt$
        \STATE $total \mathrel{+}= \exp(h_t)$
        \STATE $tt \mathrel{-}= 1 / total$
     \ENDFOR
     \FOR{$t = 1$ {\bfseries to} $\bar{m}$}
        \STATE $g[:, \Pi_{it}] \mathrel{+}= c[t] \cdot u_i$
    \ENDFOR
  \ENDFOR 
  \STATE {\bfseries Return} $g$
\end{algorithmic}
\end{algorithm}

\begin{algorithm}[tb]
  \caption{Gradient update for $V$ (Same procedure for updating $U$)}
  \label{alg:update}
\begin{algorithmic}
    \STATE {\bfseries Input:} $V, ss$, $rate$ \COMMENT{$rate$ refers to the decaying rate of the step size $ss$}
    \STATE {\bfseries Output:} $V$
    \STATE Compute gradient $g$ for $V$ \COMMENT{see alg~\ref{alg:gradv}}
    \STATE $V \mathrel{-}= ss \cdot g$
    \STATE $ss \mathrel{*}= rate$
    \STATE {\bfseries Return} $V$
\end{algorithmic}
\end{algorithm}

\begin{figure}[ht]
\vskip 0.2in
\begin{center}
\centerline{\includegraphics[width=\columnwidth]{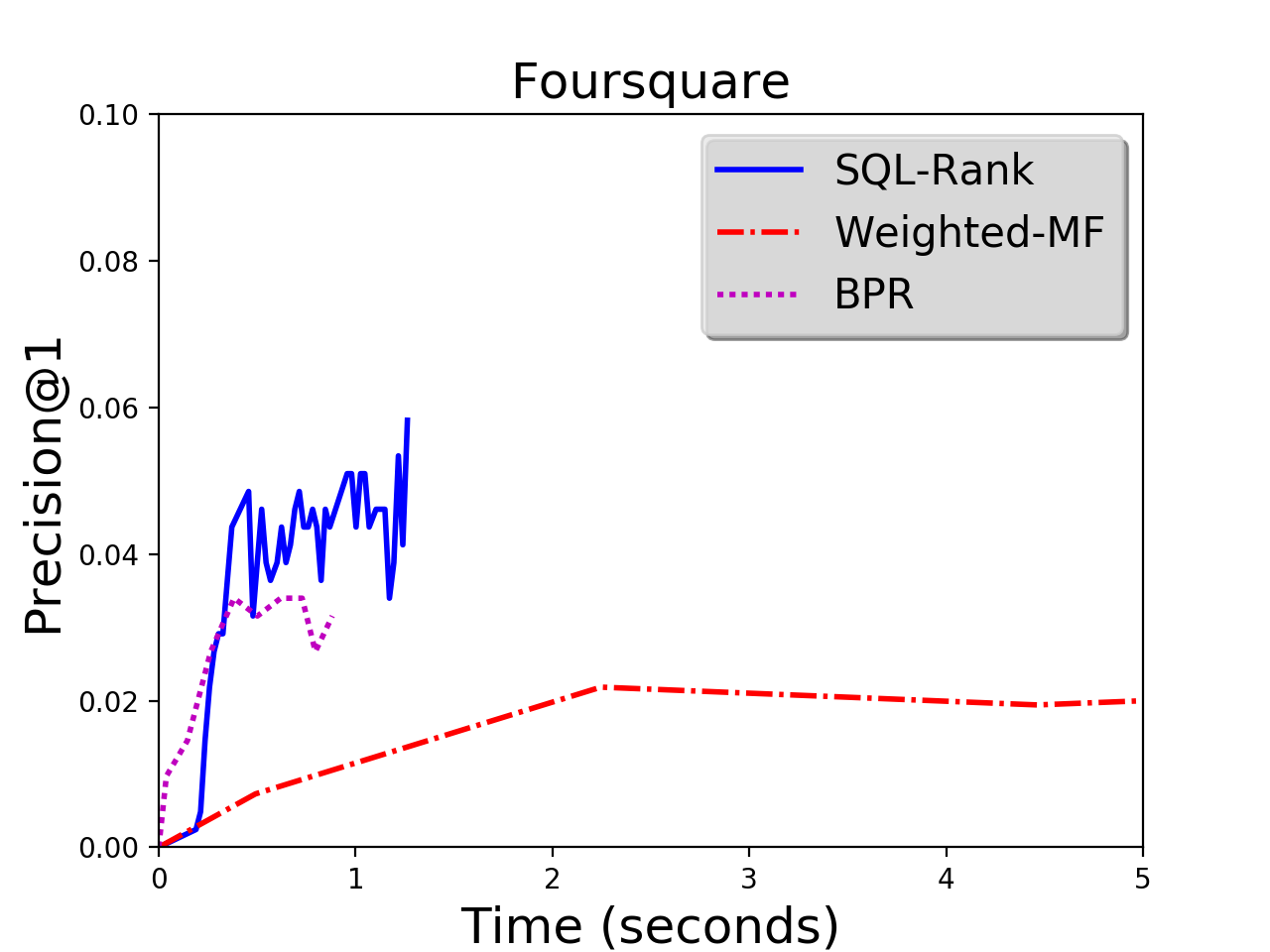}}
\caption{Comparing implicit feedback methods.}
\label{time2}
\end{center}
\vskip -0.2in
\end{figure}

\begin{figure}[ht]
\vskip 0.2in
\begin{center}
\centerline{\includegraphics[width=\columnwidth]{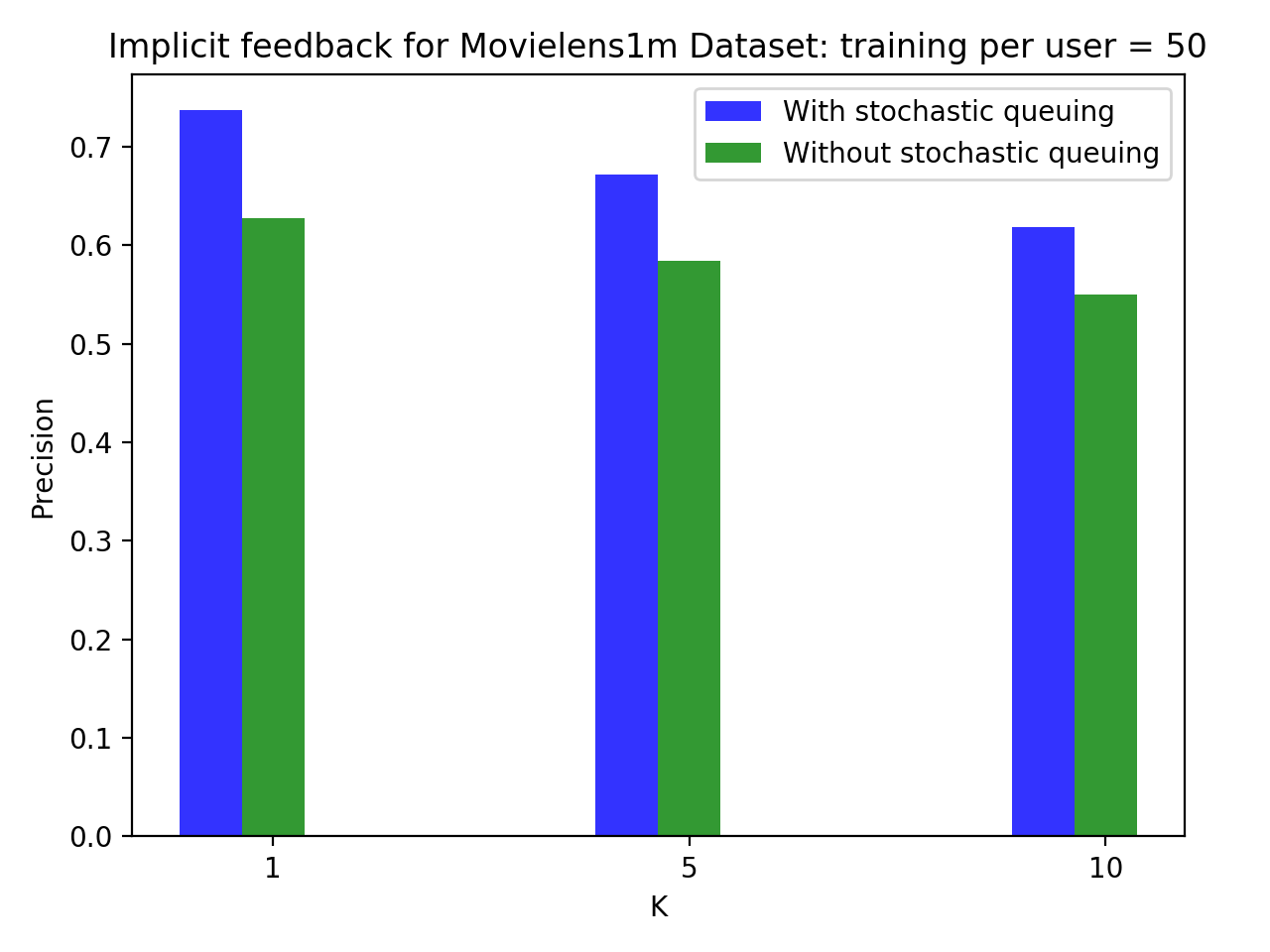}}
\caption{Effectiveness of Stochastic Queuing Process.}
\label{sqp}
\end{center}
\vskip -0.2in
\end{figure}

\begin{figure}[ht]
\vskip 0.2in
\begin{center}
\centerline{\includegraphics[width=\columnwidth]{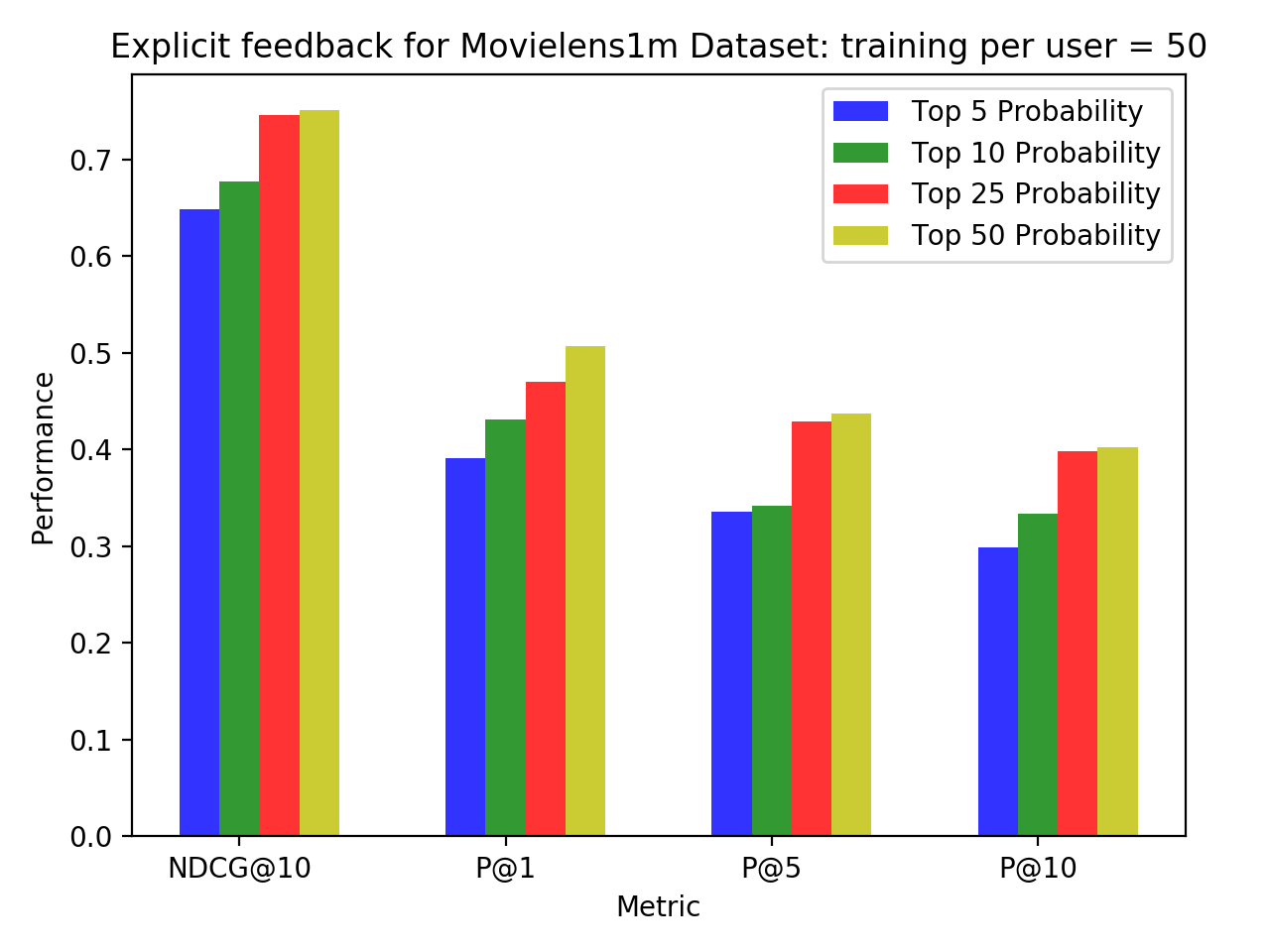}}
\caption{Effectiveness of using full lists.}
\label{explicit_topk}
\end{center}
\vskip -0.2in
\end{figure}

\end{document}